\documentclass{article}

\PassOptionsToPackage{numbers,compress}{natbib}

\usepackage[preprint]{neurips_2026}

\usepackage[utf8]{inputenc} 
\usepackage[T1]{fontenc}    
\usepackage{hyperref}       
\usepackage{url}            
\usepackage{booktabs}       
\usepackage{amsfonts}       
\usepackage{nicefrac}       
\usepackage{microtype}      
\usepackage[dvipsnames,table]{xcolor}         

\newcommand{\method}{{\textsc{TokenDrift}}}
\usepackage{amsmath}
\usepackage{comment}
\usepackage{algorithm}
\usepackage{algpseudocode}
\usepackage{amsthm}
\usepackage{mathtools} 
\usepackage{enumitem}
\usepackage{multirow}

\usepackage{tcolorbox}
\tcbuselibrary{breakable, skins}

\newtcolorbox{samplebox}[2][]{
  colback=#2!5,        
  colframe=#2!70,      
  fonttitle=\bfseries\small,
  fontupper=\small\ttfamily,
  title=#1,
  breakable,
  enhanced,
  boxsep=2pt, left=4pt, right=4pt, top=2pt, bottom=2pt,
  before skip=4pt, after skip=4pt,
}

\newtheorem{theorem}{Theorem}[section]

\newtheorem{proposition}[theorem]{Proposition}

\newtheorem{corollary}[theorem]{Corollary}

\hypersetup{
  colorlinks=true,
  linkcolor=red,
  citecolor=teal,
  urlcolor=magenta
}
\usepackage{mdframed}

\surroundwithmdframed[
  backgroundcolor=gray!5,
  linecolor=gray!50,
  linewidth=0.5pt,
  innertopmargin=6pt,
  innerbottommargin=6pt
]{proposition}

\surroundwithmdframed[
  backgroundcolor=gray!5,
  linecolor=gray!50,
  linewidth=0.5pt,
  innertopmargin=6pt,
  innerbottommargin=6pt
]{corollary}

\title{Drifting Objectives for Refining \\Discrete Diffusion Language Models}

\author{%
  Daisuke Oba\textsuperscript{1} \quad
  Hiroki Furuta\textsuperscript\quad
  Naoaki Okazaki\textsuperscript{1,2,3}\\
  \textsuperscript{1}Institute of Science Tokyo \quad
  \textsuperscript{2}AIST \quad
  \textsuperscript{3}NII LLMC\\
  \texttt{\{daisuke.oba@nlp.,okazaki@\}comp.isct.ac.jp} 
}

\begin{document}

\maketitle

\begin{abstract}
Discrete diffusion language models (DDLMs) generate text by iteratively denoising categorical token sequences, while recent drifting methods for continuous generators suggest that part of this sampling-time correction can instead be absorbed into training through an anti-symmetric fixed-point objective. We study how to transfer this principle to DDLMs, where the main challenge is the interface with discrete text: hard token samples are non-differentiable, and categorical predictions do not directly provide continuous samples to drift. 
We formulate \method, a drifting objective that lifts categorical predictions to soft-token features, applies anti-symmetric drifting in a frozen semantic space, and backpropagates the resulting stop-gradient feature target to DDLM logits.
In controlled continual-training experiments with masked and uniform-state diffusion backbones, \method\ improves fixed-NFE generation quality over matched continuation baselines, reducing Gen.-PPL at 4 NFEs by \(89\%\) on MDLM and \(86\%\) on DUO.
These results suggest that drifting can provide a practical refinement objective for DDLMs.
\par\smallskip{\centering Project page: \url{https://daioba.github.io/tokendrift/}\par}
\end{abstract}

\section{Introduction}
Discrete diffusion language models (DDLMs) provide a non-autoregressive alternative to left-to-right generation by iteratively denoising corrupted token sequences~\cite{austin2021structured,sedd,sahoo2024simple,shi2024simplified,duo,nie2025llada,ye2025dream,llada15,duo2,sahoo2026scaling}. 
Their practical behavior is therefore governed not only by the learned denoiser, but also by the quality achieved under a fixed number of denoising steps. 
While many works improve sampling by designing new schedules, samplers, or distilled generators, we ask a complementary question: \emph{can the training objective itself refine an existing DDLM so that the same sampler produces better samples at the same inference budget?}

Recent drifting-based methods for continuous generative models offer a promising training-side perspective for this question~\cite{drifting}. 
They replace part of the iterative correction normally performed during sampling with a fixed-point training objective (Fig.~\ref{fig:abstractive}; left): generated samples are moved along an attraction--repulsion field, toward nearby data samples and away from nearby model samples. 
This is appealing for DDLM refinement because it targets the generated samples themselves, rather than only the reconstruction of corrupted tokens, and therefore provides a direct objective for improving sample quality under a fixed inference budget.

However, transferring drifting to DDLMs is not a direct substitution. 
In continuous drifting, generated samples or their feature representations can be nudged along a drift direction and used as stop-gradient training targets. 
For text, the generator instead produces {categorical distributions} over tokens. 
If these distributions are collapsed to \textit{hard tokens} before feature extraction, the resulting feature-space loss no longer provides a useful gradient to the model logits. 
Thus, applying drifting to DDLMs requires a differentiable bridge from categorical predictions to the continuous feature space where the drift is defined.

\begin{figure}[h]
    \centering
    \includegraphics[width=0.9\linewidth]{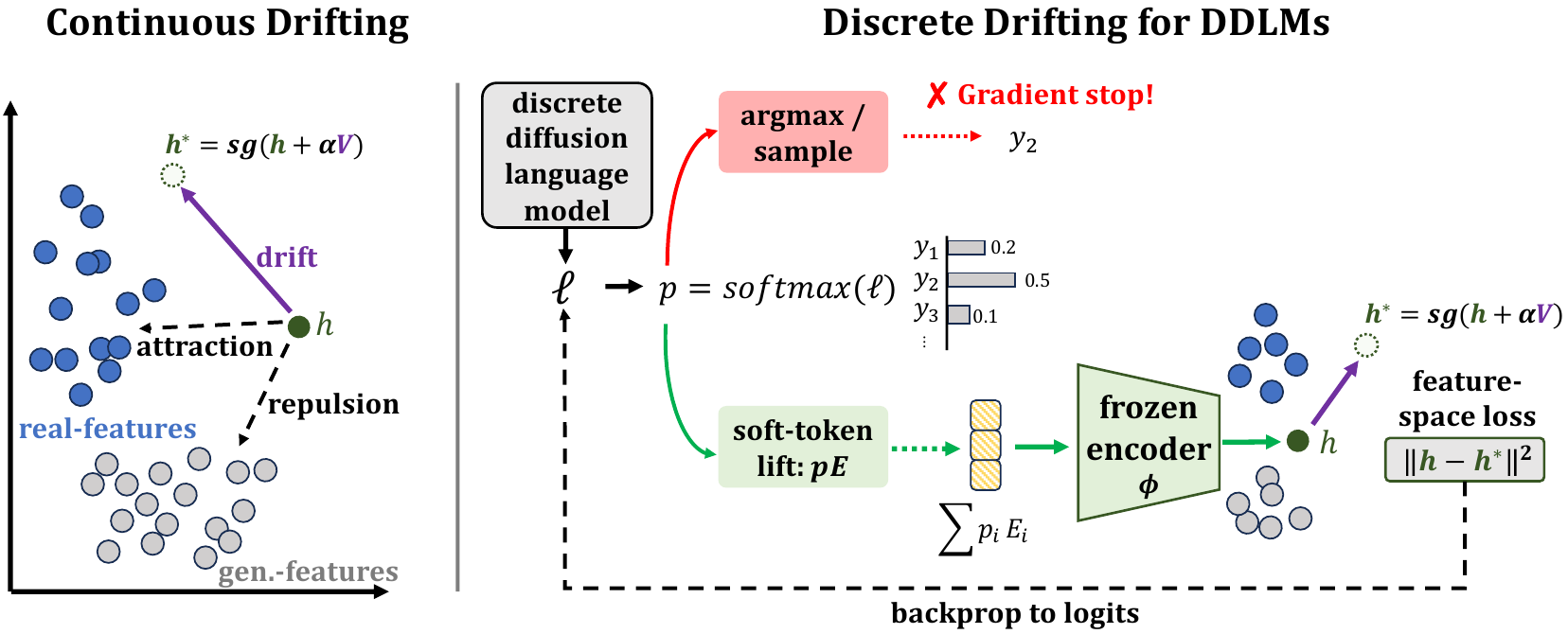}
    \caption{
    \textbf{Overview of our drifting formulation for discrete diffusion language models.}
    Original drifting constructs \textcolor{PineGreen}{a stop-gradient target $h^\star$} by moving \textcolor{PineGreen}{the generated feature $h$} along a \textcolor{purple}{drift field $V$}.
    For discrete text, hard token sampling {blocks gradients}, so we lift token probabilities to \textcolor{Goldenrod}{soft embeddings}, compute the drift target in feature space, and backpropagate the loss to logits.}
    \label{fig:abstractive}
\end{figure}

We formulate \method, a drifting objective for refining DDLMs through soft-token features (Fig.~\ref{fig:abstractive}; right). 
Rather than sampling hard tokens, \method\ feeds the frozen semantic encoder with expected token embeddings computed from the model's categorical predictions. 
This creates a differentiable path from the feature-space drifting loss back to the DDLM logits. 
In that feature space, we estimate the same attraction--repulsion drift as in continuous drifting and train the generator toward a stop-gradient drifted feature target.

A central part of our study is formulation. 
There are several plausible ways to connect a feature-space drift to a categorical generator: one can match the drifted target directly in feature space, convert the drift into a mirror teacher in logit space, or combine either objective with the original denoising loss. 
We compare these alternatives under matched budgets in Section~\ref{sec:experiments:formulation},
and find that direct feature-space drifting provides the most stable trade-off between likelihood-based quality and entropy. 
We further show that the soft-token lift is essential: replacing it with a hard straight-through token surrogate~\cite{bengio2013estimating} severely degrades generation quality, indicating that preserving predictive uncertainty in the semantic feature space is important for effective drifting.

In controlled continual-training experiments on OpenWebText (OWT)~\cite{owt}, \method\ substantially improves fixed-NFE generation quality over both a pretrained masked diffusion language model, MDLM~\cite{sahoo2024simple}, and ordinary continuation training. 
The same objective also improves a uniform-state discrete diffusion backbone, DUO~\cite{duo}, suggesting that the effect is not specific to masked diffusion. 
Together, these results position \method\ as a refinement objective for existing DDLMs: rather than designing a new sampler or distilling a low-NFE student, we keep the model class and sampler fixed and improve the generator through the training objective.

\textbf{Contributions.}
In summary, we formulate drifting as a trainable refinement objective for DDLMs, identify the design choices needed to make it effective for discrete text, and show controlled improvements on both masked and uniform-state diffusion language models.

\section{Preliminaries}
\label{sec:preliminaries}
We briefly review the ingredients needed to formulate drifting objectives for refining DDLMs: categorical denoising models, drifting-based fixed-point learning, and the difficulty of applying continuous drifting directly to text.

\subsection{Discrete Diffusion Language Models}

Let $\mathcal V$ be a vocabulary and let $x=(x_1,\dots,x_L)\in\mathcal V^L$ be a token sequence. 
A DDLM parameterized by $\theta$ defines a corruption process over token sequences and trains a denoiser $f_\theta$ to predict clean tokens from corrupted inputs. 
At each denoising step, the model outputs logits
\(
\ell_\theta=f_\theta(x_t)\in\mathbb{R}^{L\times|\mathcal V|},
\)
which induce position-wise categorical distributions
\(
p_{\theta,t}=\mathrm{softmax}(\ell_{\theta,t})\in\Delta^{|\mathcal V|-1}.
\)

Different DDLMs instantiate the corruption process differently: masked diffusion corrupts tokens toward a mask state, while uniform-state diffusion corrupts tokens toward a uniform distribution over the vocabulary. 
In both cases, the denoiser predicts categorical token distributions and generation proceeds by applying the denoiser for a chosen number of steps.

\subsection{Drifting-Based Fixed-Point Learning}
Drifting defines a fixed-point training rule for continuous generators~\cite{drifting}. 
For a generated sample or feature $y$, let
\(
V(y;P_{\mathrm{data}},P_{\mathrm{\theta}})
\)
be a drift field estimated by attraction toward nearby data samples and repulsion from nearby model samples. 
Drifting forms a stop-gradient target
\[
y^\star=\mathrm{sg}\!\left(y+\alpha V(y;P_{\mathrm{data}},P_{\mathrm{\theta}})\right),
\]
and trains the generator to match this target. 
When raw sample-space distances are not meaningful, the same construction can be applied in a frozen feature space.

A key structural property is anti-symmetry: swapping the attractive and repulsive distributions reverses the drift direction. 
Thus, if the model and data distributions coincide, then attraction and repulsion cancel and
\(
V(\cdot;P,P)=0,
\)
so the drift signal vanishes at equilibrium. 
We aim to preserve this fixed-point structure for DDLMs.

\subsection{The Discrete-Text Interface}
Continuous drifting assumes generated samples or features that can be additively shifted and differentiated through. 
DDLMs instead output categorical token distributions: hard tokenization breaks gradients to logits, while direct probability updates must respect the simplex. 
Thus, applying drifting to DDLMs requires a differentiable bridge from categorical predictions to the continuous feature space where drift is defined.

\section{\method}
\label{sec:method}
We formulate drifting as a refinement objective for DDLMs. 
Given categorical token predictions, \method\ lifts them to soft-token features, computes an anti-symmetric attraction--repulsion drift in a frozen semantic space, and trains the generator toward a stop-gradient drifted feature target. 
This makes feature-space drifting differentiable with respect to token logits.

\textbf{Notation.}
We use $i$ for samples, $t$ for token positions, and $v$ for vocabulary indices. 
For sample $i$, the generator outputs logits 
\(\ell_i\!\in\!\mathbb{R}^{L\times|\mathcal V|}\) 
and distributions 
\(p_i\!\!=\!\!\mathrm{softmax}(\ell_i)\!\in\!\mathbb{R}^{L\times|\mathcal V|}\), 
with \(p_{i,t}\in\Delta^{|\mathcal V|-1}\).

\subsection{Soft-Token Feature Lift}
\label{subsec:soft_token_lift}

Let $E\in\mathbb{R}^{|\mathcal V|\times d}$ be the embedding matrix used by a frozen semantic encoder $\phi$, and $x_0^{(i)}$ be a clean sequence and $\bar x^{(i)}$ its corrupted input. 
Given $\bar x^{(i)}$, the generator outputs distributions $p_i$. 
We form encoder inputs by setting $\tilde e_{i,t}=p_{i,t}E$ for predicted positions $t\in\mathcal M_i$ and $\tilde e_{i,t}=E[\bar x^{(i)}_t]$ otherwise, where $\mathcal M_i$ is defined by the underlying diffusion backbone.
The generated feature and the corresponding real feature are
\[
h_i=\phi(\tilde e_i)\in\mathbb{R}^m,
\qquad
u_i=\phi(E[x_0^{(i)}])\in\mathbb{R}^m.
\]
Thus, $h_i$ is the feature of the model-completed sequence, while $u_i$ is the feature of the corresponding clean sequence. 
Because $\tilde e_i$ depends on $\ell_i$ through $p_i$ on predicted positions, feature-space drifting losses can backpropagate to the logits. 
For backbones without observed-token positions, such as uniform-state diffusion, 
we set $\mathcal M_i=\{1,\dots,L\}$ and use soft embeddings at all positions.

\subsection{Anti-Symmetric Drift Estimation}
\label{subsec:drift_estimation}

For each generated feature $h_i$, we build a positive reference set $\mathcal P_i$ from real data features and a negative reference set $\mathcal N_i$ from generated features. 
The drift follows the attraction--repulsion structure of drifting models~\cite{drifting}: positives pull the sample toward the data distribution, while negatives push it away from the current model distribution.

For $u_j\in\mathcal P_i$ and $v_k\in\mathcal N_i$, we compute temperature-scaled affinities
\(s^+_{ij}=-\|h_i-u_j\|_2^2/\tau\) and
\(s^-_{ik}=-\|h_i-v_k\|_2^2/\tau\).
Following the original drifting construction, we jointly normalize positive and negative affinities to obtain weights $W^+_{ij}$ and $W^-_{ik}$.
The corresponding positive and negative barycenters define the temperature-\(\tau\) drift field:
\[
b_i^+ = \sum\nolimits_{j} W^+_{ij}u_j,
\qquad
b_i^- = \sum\nolimits_k W^-_{ik}v_k,
\qquad
V_i^{(\tau)} = b_i^+ - b_i^- .
\]
Thus, the drift points from nearby model features toward nearby data features.

\textbf{Multi-temperature drift.}
We compute $V_i^{(\tau)}$ for each $\tau\in\mathcal T$, normalize each temperature by a scalar batch-level RMS scale, and average:
\[
s^{(\tau)}
=
\sqrt{\operatorname{mean}_i\|V_i^{(\tau)}\|_2^2+\epsilon},
\qquad
V_i
=
\frac{1}{|\mathcal T|}
\sum\nolimits_{\tau\in\mathcal T}
\frac{V_i^{(\tau)}}{s^{(\tau)}} .
\]
This prevents any single temperature scale from dominating the final drift.

\subsection{Feature-Space Fixed-Point Objective}
\label{subsec:feature_objective}

Our main objective is a direct feature-space fixed-point loss. Given the current generated feature $h_i$ and drift $V_i$, we form the stop-gradient target
\(
h_i^\star
=
\mathrm{sg}\!\left(h_i+\alpha V_i\right),
\)
where $\alpha>0$ is a drift scale. The drifting objective is
\[
\mathcal L_{\mathrm{drift}}
=
\frac{1}{2B}
\sum\nolimits_{i=1}^B
\left\|
h_i - h_i^\star
\right\|_2^2.
\]
Because $h_i^\star$ is frozen, the feature-space gradient is proportional to $-V_i$, so minimizing this loss pushes $h_i$ in the drift direction. 
Through the soft-token lift, this feature-space signal backpropagates to the token logits.

\textbf{Relation to the base objective.}
The drifting objective can be used either alone or in combination with the original discrete diffusion training objective. 
Unless otherwise stated, our main method uses the drifting objective as the primary training signal, while the combined variant is evaluated as an ablation (Sec.~\ref{sec:experiments:formulation}).

\subsection{Alternative Formulation via Mirror Teacher}
\label{subsec:mirror_alternative}

As an alternative to direct feature-space matching, we also consider converting the feature-space drift into a token-level teacher distribution. 
Let $h_i=h(\ell_i)$ be the feature induced by the logits. 
We define
\[
g_i=\nabla_{\ell_i}\!\left(h(\ell_i)^\top\mathrm{sg}(V_i)\right),
\qquad
\ell_i^\star=\mathrm{sg}(\ell_i+\eta g_i),
\qquad
p_i^\star=\mathrm{softmax}(\ell_i^\star).
\]
The softmax maps the updated logits back to valid categorical distributions, yielding a simplex-aware mirror teacher.
We evaluate two matching losses on predicted positions $\mathcal M_i$: distributional KL and logit-space MSE:
\[
\mathcal L_{\mathrm{mirror\text{-}KL}}
=
\frac{1}{B}\sum\nolimits_i\sum\nolimits_{t\in\mathcal M_i}
\mathrm{KL}(p_{i,t}^\star\|p_{i,t}),
\qquad
\mathcal L_{\mathrm{mirror\text{-}MSE}}
=
\frac{1}{B}\sum\nolimits_i\sum\nolimits_{t\in\mathcal M_i}
\|\ell_{i,t}^\star-\ell_{i,t}\|_2^2.
\]
We use these mirror objectives as alternative formulations in the study of Sec.~\ref{sec:experiments:formulation}.
Appendix~\ref{app:mirror_theory} provides the corresponding theoretical justification: 
the mirror teacher is the KL-proximal simplex update induced by the logit-space direction (Prop.~\ref{prop:mirror_teacher}), 
this direction locally improves alignment with the feature-space drift (Prop.~\ref{prop:mirror_local_ascent}), 
and the construction preserves the equilibrium property of anti-symmetric drifting (Prop.~\ref{prop:mirror_equilibrium}).

\subsection{Algorithm}
\label{subsec:training_algorithm}

Algo.~\ref{alg:drifting_objective} summarizes one training step of \method, our main feature-space drifting objective. 
This applies to different DDLM backbones; the corruption process and the set of predicted positions $\mathcal M_i$ are backbone-specific. 
Here $\operatorname{gather}(\cdot)$ denotes cross-device feature gathering when enabled.

\begin{algorithm}[H]
\caption{\textbf{One training step of \method}}
\label{alg:drifting_objective}
\small
\begin{algorithmic}[1]
\Require clean mini-batch $x_0^{1:B}$, corruption process $q(\bar x\!\!\mid\!\! x_0)$,  frozen encoder $\phi$, queues $\mathcal Q_{\mathrm{real}},\mathcal Q_{\mathrm{gen}}$
\State sample corrupted inputs $\bar x^{(i)}\sim q(\bar x\mid x_0^{(i)})$ for all $i=1,\dots,B$
\State compute logits $\ell_i=f_\theta(\bar x^{(i)})$ and distributions $p_i=\mathrm{softmax}(\ell_i)$ for all $i$
\State form encoder inputs $\tilde e_i$ using $p_{i,t}E$ for $t\in\mathcal M_i$ and $E[\bar x_t^{(i)}]$ otherwise
\State compute generated and real features $h_i=\phi(\tilde e_i)$ and $u_i=\phi(E[x_0^{(i)}])$ for all $i$
\State form current feature sets $\mathcal U_{\mathrm{cur}}=\operatorname{gather}(\{u_i\}_{i=1}^B)$ and $\mathcal H_{\mathrm{cur}}=\operatorname{gather}(\{h_i\}_{i=1}^B)$
\State build references $\mathcal U=\mathcal U_{\mathrm{cur}}\cup\mathcal Q_{\mathrm{real}}$ and $\mathcal H=\mathcal H_{\mathrm{cur}}\cup\mathcal Q_{\mathrm{gen}}$
\State compute $V_i=\mathrm{Drift}_{\mathcal T}\!\left(h_i,\mathrm{sg}(\mathcal U),\mathrm{sg}(\mathcal H\setminus\{h_i\})\right)$ for all $i$
\State construct feature targets $h_i^\star=\mathrm{sg}(h_i+\alpha V_i)$ for all $i$
\State compute $\mathcal L_{\mathrm{drift}}=\frac{1}{2B}\sum_{i=1}^B\|h_i-h_i^\star\|_2^2$
\State backpropagate through $\mathcal L_{\mathrm{drift}}$ and update $\theta$
\State push detached current features $\mathrm{sg}(\{u_i\}_{i=1}^B)$/$\mathrm{sg}(\{h_i\}_{i=1}^B)$ into $\mathcal Q_{\mathrm{real}}$/$\mathcal Q_{\mathrm{gen}}$, evicting the oldest entries
\end{algorithmic}
\end{algorithm}

\textbf{Inference}
At inference time, \method\ simply uses the original DDLM sampler. 
All drift-related components are training-time only and add no sampling-time cost.

\subsection{Theoretical Properties of the \method\ Objective}
\label{subsec:objective_properties}

We summarize the main properties of the proposed objective; formal statements and proofs are given in Appendix~\ref{sec:theory}. 
First, the soft-token lift makes the feature-space loss trainable for categorical generators (Prop.~\ref{prop:soft_token_lift}). 
Since
\(
\ell_i \rightarrow p_i=\mathrm{softmax}(\ell_i) \rightarrow \tilde e_i=p_iE \rightarrow h_i=\phi(\tilde e_i)
\)
is differentiable on predicted positions, gradients from the drifting loss can backpropagate to the logits. 
Hard token selection (e.g., $\mathrm{argmax}$) would break this path and turn the feature-space target into a non-differentiable signal.

Second, the fixed-point loss does more than match features: it induces the drift-following update prescribed by the computed field (Prop.~\ref{prop:drift_gradient}). 
For the target \(h_i^\star=\mathrm{sg}(h_i+\alpha V_i)\),
\[
\nabla_{h_i}\frac12\|h_i-h_i^\star\|_2^2=-\alpha V_i,
\qquad
\nabla_{\ell_i}\mathcal L_{\mathrm{drift}}
=
-\alpha J_{h_i}(\ell_i)^\top V_i .
\]
Thus, the continuous feature-space drift is pulled back through the soft-token feature map to produce a logit-space update direction for the DDLM.

Third, the objective inherits the equilibrium structure of drifting. 
If the attraction--repulsion field is anti-symmetric, then the drift vanishes when the data and model feature distributions match (Cor.~\ref{cor:equilibrium}).

These properties do not imply global convergence of the nonconvex generator, but they show that the proposed objective is differentiable, drift-following, and fixed-point consistent in the geometry where the drift is defined.


\section{Experiments}
\label{sec:experiments}

We evaluate whether drifting objectives improve discrete diffusion language models under matched training and inference budgets. 
We conduct continual training on OpenWebText (OWT)~\cite{owt} across two DDLM parameterizations: masked diffusion and uniform-state diffusion. 
For each parameterization, we start from a released checkpoint and compare against ordinary continuation training while keeping the backbone, initialization, data, additional compute budget, and inference budget matched. 
Our experiments focus on two questions: whether drifting improves generation quality beyond extra optimization alone, and which discrete adaptation of drifting is most effective.

\subsection{Experimental Setup}
\label{subsec:exp_setup}

We provide full implementation details in Appendix~\ref{app:exp_details}; here we summarize the controlled setup used in the main experiments.

\textbf{Benchmarks and backbones.}
Our main benchmark is OpenWebText (OWT)~\cite{owt}, where we perform continual training from released DDLM checkpoints. 
We consider two discrete diffusion parameterizations: 
masked diffusion, \textbf{MDLM}~\cite{sahoo2024simple}, using the released 170M MDLM checkpoint, and uniform-state diffusion, \textbf{DUO}~\cite{duo}, using the released checkpoint. 
In all controlled comparisons, methods within the same backbone start from the same checkpoint and use the same data, additional training budget, and evaluation protocol.

\textbf{Controlled comparisons.}
For each backbone, we compare the released checkpoint, ordinary continuation training with the original DDLM objective, and \method. 
These baselines isolate whether improvements come from the drifting objective rather than from additional optimization alone. 
We also evaluate key components of drifting formulations (Sec.~\ref{subsec:mirror_alternative}), including mirror-teacher variants, to study which discrete adaptation of drifting is most effective.



\textbf{Drifting implementation.}
\method\ uses a frozen copy of the same pretrained backbone checkpoint as the semantic encoder $\phi$ to compute pooled sequence features.
It estimates an anti-symmetric attraction--repulsion drift from real and generated references built from the current distributed micro-batch and FIFO queues of detached features; unless otherwise stated, the default queue size is $1024$ for both real and generated features.
The generator is trained with the feature-space drifting objective, and the drift is estimated using the multi-temperature scheme of the original drifting method~\cite{drifting} with $\mathcal T=\{0.02,0.05,0.2\}$.

\textbf{Evaluation.}
We focus on unconditional text generation and report generative perplexity (Gen.-PPL) computed by GPT-2 Large~\cite{gpt2}. 
We also report entropy as a diagnostic for diversity and degeneration. 
All main comparisons use matched numbers of function evaluations (NFEs).

\begin{table}[t]
  \centering
  \small
  \setlength{\tabcolsep}{4pt}
\caption{
\textbf{OWT few-step generation results for masked diffusion (MDLM) at ckpt. 13k.}
Non-gray rows are controlled comparisons using the same additional training budget; values are mean$_{\pm\mathrm{SD}}$ over 3 seeds.
Gray rows provide contextual results from distillation methods and are not controlled baselines.
Entropy is reported as a diversity diagnostic, and bold denotes the best Gen.-PPL per NFE.
}
  \begin{tabular}{l @{\,\,\,} r@{\,\,\,}r r@{\,\,\,}r r@{\,\,\,}r}
    \toprule
    & \multicolumn{2}{c}{{$\text{NFE}=4$}}
    & \multicolumn{2}{c}{{$\text{NFE}=8$}}
    & \multicolumn{2}{c}{{$\text{NFE}=16$}} \\
    \cmidrule(lr){2-3}\cmidrule(lr){4-5}\cmidrule(lr){6-7}
    Method
      & {Gen.-PPL}$\downarrow$ & {Entropy}$\uparrow$
      & {Gen.-PPL}$\downarrow$ & {Entropy}$\uparrow$
      & {Gen.-PPL}$\downarrow$ & {Entropy}$\uparrow$ \\
    \midrule
    MDLM~\cite{sahoo2024simple}
      & $1942.33_{\pm 66.80}$ & $5.95_{\pm 0.02}$
      & $799.52_{\pm 24.28}$  & $5.88_{\pm 0.02}$
      & $338.81_{\pm 15.71}$  & $5.82_{\pm 0.02}$ \\
    \midrule
    Cont.\ learning
      & $1943.95_{\pm 86.65}$ & $5.92_{\pm 0.02}$
      & $787.24_{\pm 4.18}$   & $5.85_{\pm 0.01}$
      & $337.85_{\pm 11.45}$  & $5.78_{\pm 0.02}$ \\
    \textbf{\method}
      & $\mathbf{208.87}_{\pm 11.51}$ & $5.41_{\pm 0.01}$
      & $\mathbf{55.04}_{\pm 1.96}$   & $5.30_{\pm 0.02}$
      & $\mathbf{27.35}_{\pm 0.20}$            & $5.16_{\pm 0.02}$  \\
    \midrule
    \textcolor{black!55}{\textit{Distillation}} \\
    \textcolor{black!55}{ ~~ SDTT~\cite{SDTT}}
      & \textcolor{black!55}{$358.67_{\pm 17.14}$} & \textcolor{black!55}{$6.30_{\pm 0.04}$}
      & \textcolor{black!55}{$127.64_{\pm 4.41}$}  & \textcolor{black!55}{$6.41_{\pm 0.02}$}
      & \textcolor{black!55}{$65.91_{\pm 6.91}$}   & \textcolor{black!55}{$6.42_{\pm 0.06}$} \\
    \textcolor{black!55}{ ~~ Di4C~\cite{di4c}}
      & \textcolor{black!55}{$246.89_{\pm 8.98}$} & \textcolor{black!55}{$6.32_{\pm 0.03}$}
      & \textcolor{black!55}{$84.20_{\pm 3.14}$}  & \textcolor{black!55}{$6.41_{\pm 0.00}$}
      & \textcolor{black!55}{$44.66_{\pm 5.19}$}  & \textcolor{black!55}{$6.35_{\pm 0.08}$} \\
    \textcolor{black!55}{ ~~ DiDi-Inst.~\cite{didi-instruct}}
      & \textcolor{black!55}{$357.39_{\pm 18.92}$} & \textcolor{black!55}{$5.24_{\pm 0.04}$}
      & \textcolor{black!55}{$83.65_{\pm 8.94}$}   & \textcolor{black!55}{$4.99_{\pm 0.03}$}
      & \textcolor{black!55}{$33.23_{\pm 2.14}$}   & \textcolor{black!55}{$4.75_{\pm 0.04}$} \\
    \bottomrule
  \end{tabular}
  \label{tab:gen_ppl_entropy}
\end{table}

\begin{table}[t]
  \centering
  \small
  \setlength{\tabcolsep}{4pt}
  \caption{\textbf{OWT few-step generation results for uniform-state diffusion (DUO) at ckpt. 13k.} Non-gray rows are controlled comparisons using the same additional training budget; values are mean$_{\pm\mathrm{SD}}$ over 3 seeds. The gray row provides a contextual DCD reference and is not a controlled baseline. Entropy is reported as a diversity diagnostic, and bold denotes the best Gen.-PPL per NFE.}
  \begin{tabular}{l rr rr rr}
    \toprule
    & \multicolumn{2}{c}{{$\text{NFE}=4$}}
    & \multicolumn{2}{c}{{$\text{NFE}=8$}}
    & \multicolumn{2}{c}{{$\text{NFE}=16$}} \\
    \cmidrule(lr){2-3}\cmidrule(lr){4-5}\cmidrule(lr){6-7}
        Method
      & {Gen.-PPL}$\downarrow$ & {Entropy}$\uparrow$
      & {Gen.-PPL}$\downarrow$ & {Entropy}$\uparrow$
      & {Gen.-PPL}$\downarrow$ & {Entropy}$\uparrow$ \\
    \midrule
    DUO~\cite{duo}
      & $508.96_{\pm 40.38}$ & $5.54_{\pm 0.05}$
      & $196.94_{\pm 10.79}$ & $5.57_{\pm 0.02}$
      & $114.84_{\pm 5.84}$  & $5.58_{\pm 0.02}$ \\
    \midrule
    Cont.\ learning
      & $539.28_{\pm 43.97}$ & $5.54_{\pm 0.05}$
      & $199.31_{\pm 14.61}$ & $5.57_{\pm 0.01}$
      & $114.55_{\pm 3.59}$  & $5.58_{\pm 0.03}$ \\
    \textbf{\method}
      & $\mathbf{76.17}_{\pm 3.18}$ & $5.40_{\pm 0.05}$
      & $\mathbf{54.98}_{\pm 1.40}$ & $5.42_{\pm 0.05}$
      & $\mathbf{47.36}_{\pm 0.94}$ & $5.43_{\pm 0.04}$ \\
    \midrule
    \textcolor{black!55}{\textit{Distillation}:~DCD~\cite{duo}}
      & \textcolor{black!55}{$299.67_{\pm 28.79}$} & \textcolor{black!55}{$5.51_{\pm 0.04}$}
      & \textcolor{black!55}{$123.48_{\pm 1.79}$}  & \textcolor{black!55}{$5.55_{\pm 0.03}$}
      & \textcolor{black!55}{$81.27_{\pm 3.75}$}   & \textcolor{black!55}{$5.54_{\pm 0.02}$} \\
    \bottomrule
  \end{tabular}
  \label{tab:gen_ppl_entropy_duo}
\end{table}

\begin{figure}[t]
    \centering
    \includegraphics[width=0.9\linewidth]{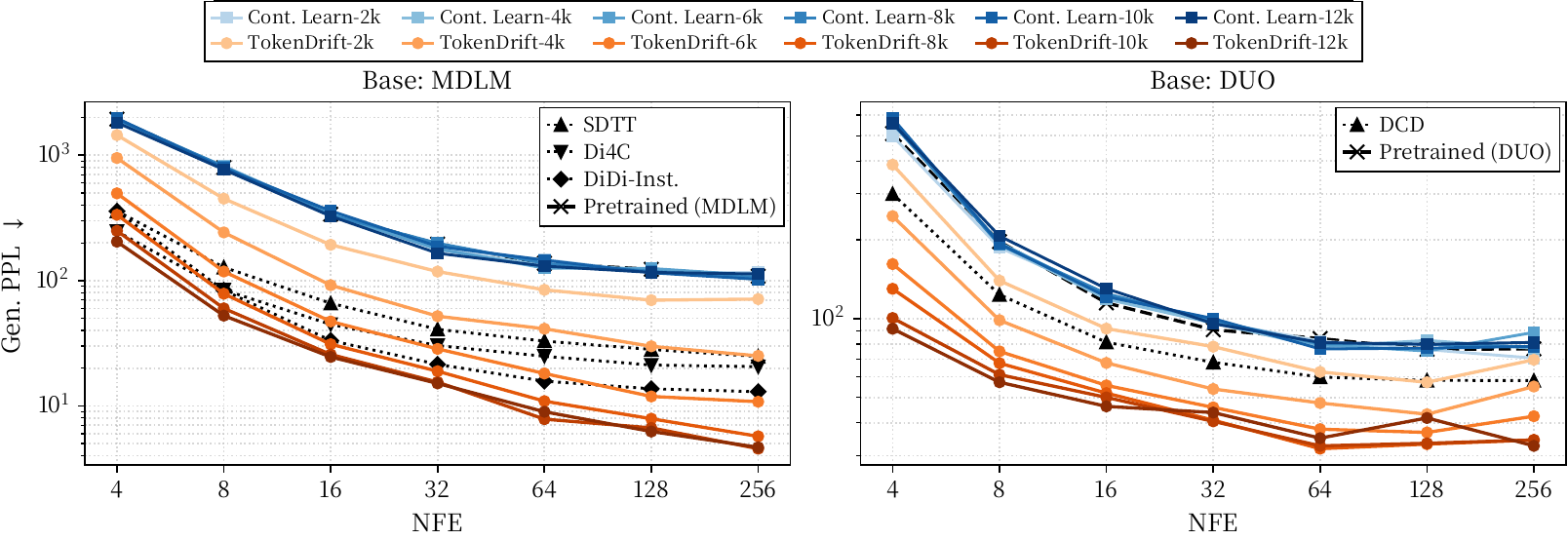}
    \caption{\textbf{Training dynamics from released MDLM~\cite{sahoo2024simple} and DUO~\cite{duo} checkpoints.} As drifting training progresses, Gen.-PPL decreases across the NFEs, showing that our drifting objective progressively improves fixed-budget generation quality rather than merely selecting a better final checkpoint.}
    \label{fig:dynamics}
\end{figure}

\subsection{Main Results}
\label{subsec:results}

\textbf{Masked diffusion.}
Table~\ref{tab:gen_ppl_entropy} reports controlled OWT continual-training results on top of MDLM in the few-step regime, using 4, 8, and 16 NFEs. 
The controlled block keeps the checkpoint, backbone, data, training budget, and evaluation protocol fixed across MDLM, ordinary continuation training, and \method. 
Under this matched setting, \method\ substantially improves Gen.-PPL at every reported NFE. 
Ordinary continuation training barely changes the original checkpoint, indicating that the gains come from the drifting objective rather than extra optimization alone. 
Within this few-step range, entropy decreases moderately but remains far from the severe collapse.

\textbf{Uniform-state diffusion.}
Table~\ref{tab:gen_ppl_entropy_duo} evaluates the same few-step regime on DUO, a uniform-state discrete diffusion backbone. 
Starting from the released DUO checkpoint and using the same additional training budget, \method\ again improves Gen.-PPL over both the original checkpoint and ordinary continuation training at every reported NFE. 
The entropy values remain close to the original DUO and continuation baselines in this range, suggesting that the quality gains do not come from an obvious diversity collapse. 
Together with the MDLM results, this shows that the drifting objective is effective across both masked and uniform-state DDLM parameterizations.

\textbf{Progressive improvement from drifting.}
Figure~\ref{fig:dynamics} tracks Gen.-PPL during the optimization. 
As the drifting objective is optimized, Gen.-PPL decreases across the evaluated NFEs, and the gap to both the original checkpoint and ordinary continuation training widens over time. 
This is the key empirical signal: a drifting loss consistently translates into better token-level sample quality under fixed decoding budgets via \method.

\textbf{Contextual distillation references.}
The gray rows in Tables~\ref{tab:gen_ppl_entropy} and~\ref{tab:gen_ppl_entropy_duo} provide contextual comparisons to specialized distillation methods. 
These methods use different objectives, training pipelines, and sometimes different teachers or students, so they are not controlled baselines for our claims. 
Nevertheless, \method\ achieves lower Gen.-PPL than the reported references at all listed NFEs.

\begin{table}[t]
  \centering
  \small
  \setlength{\tabcolsep}{4pt}
  
  \caption{\textbf{Ablation on the input to the encoder $\phi$ for \method\
           on top of MDLM at ckpt 13k.}
           \emph{Soft} feeds the expected embedding $p_t E$ (continuous mixture over the
           vocabulary), while \emph{Hard} feeds the embedding of the argmax token
           $E[\arg\max p_t]$.
           $^{*}$default configuration used in the main table.}
  \begin{tabular}{l rr rr rr}
    \toprule
    & \multicolumn{2}{c}{{$\text{NFE}=4$}}
    & \multicolumn{2}{c}{{$\text{NFE}=8$}}
    & \multicolumn{2}{c}{{$\text{NFE}=16$}} \\
    \cmidrule(lr){2-3}\cmidrule(lr){4-5}\cmidrule(lr){6-7}
    {Input of $\phi$}
      & {Gen.-PPL$\downarrow$} & {Entropy$\uparrow$}
      & {Gen.-PPL$\downarrow$} & {Entropy$\uparrow$}
      & {Gen.-PPL$\downarrow$} & {Entropy$\uparrow$} \\
    \midrule
    Soft: $p_t E$$^{*}$
      & $\mathbf{208.87}_{\pm 11.51}$ & $5.41_{\pm 0.01}$
      & $\mathbf{55.04}_{\pm 1.96}$   & $5.30_{\pm 0.02}$
      & $\mathbf{27.35}_{\pm 0.20}$            & $5.16_{\pm 0.02}$  \\
    Hard: $y_t^{\mathrm{ST}} E$
      & $6464.88_{\pm 3398.77}$ & $1.56_{\pm 0.03}$
      & $1036.62_{\pm 862.46}$  & $1.37_{\pm 0.12}$
      & $447.63_{\pm 200.38}$   & $1.61_{\pm 0.11}$ \\
    \bottomrule
  \end{tabular}
  \label{tab:phi_input_ablation}
\end{table}

\subsection{Formulation Study: How should Drifting Be Applied to Discrete Diffusion LMs?}\label{sec:experiments:formulation}

\textbf{Soft-token lift is essential.}
We explore whether our soft-token lift can be replaced by a hard-token straight-through estimator~\cite{bengio2013estimating}. 
The soft variant feeds the encoder the expected embedding $p_tE$, while the hard variant uses the argmax embedding in the forward pass and passes gradients through
\[
y_t^{\mathrm{ST}}
=
y_t^{\mathrm{hard}}-\mathrm{sg}(p_t)+p_t.
\qquad
y_t^{\mathrm{hard}}=\mathrm{onehot}(\arg\max_v p_{t,v}).
\]
Table~\ref{tab:phi_input_ablation} shows that, despite this surrogate gradient path, the hard variant performs much worse and collapses entropy. 
Thus, differentiability alone is not sufficient: effective drifting requires exposing the semantic encoder to the model's soft predictive distribution.

\textbf{Target space: feature or logit?}
We compare three drift objectives: direct feature-space L2 matching, 
and its alternatives, i.e., logit-space L2/KL matching through a mirror teacher (Sec.~\ref{subsec:mirror_alternative}). 
Table~\ref{tab:method_ablation} (upper) shows that all three variants are effective, indicating that the main gain comes from the drifting signal itself rather than a single target construction. 
However, they trade off likelihood and entropy differently.
We therefore use feature-space L2 as the default, stable formulation for \method.

\textbf{Interaction with the base denoising objective.}
The lower block of Table~\ref{tab:method_ablation} shows that adding the original MDLM loss is not consistently beneficial. 
For logit-space objectives, the base loss largely cancels the effect of drifting; for feature-space L2, it remains competitive but does not improve over drift-only training. 
This suggests that, in our refinement setting, drifting is not merely a regularizer on top of denoising but can serve as the primary training signal.

\begin{table}[t]
  \centering
  \small
  \setlength{\tabcolsep}{4pt}
  \caption{\textbf{Design choices ablation of \method\ at ckpt 13k.}
           Space: ``\emph{feature}'' compares teacher--student hidden representations;
           ``\emph{logit}'' compares output distributions in the mirror space.
           Dist.: L2/KL means the matching metrics between anchor and teacher point.
           Values are mean$_{\pm \mathrm{SD}}$ over 3 seeds.
           Bold = best Gen.-PPL per NFE.
           $^{*}$default configuration used in the main table.}
  \begin{tabular}{@{}c@{\,\,}c@{\,\,}c rr rr rr@{}}
    \toprule
    \multicolumn{3}{c}{{Configuration}}
    & \multicolumn{2}{c}{{$\text{NFE}=4$}}
    & \multicolumn{2}{c}{{$\text{NFE}=8$}}
    & \multicolumn{2}{c}{{$\text{NFE}=16$}} \\
    \cmidrule(lr){1-3}
    \cmidrule(lr){4-5}\cmidrule(lr){6-7}\cmidrule(lr){8-9}
    {Space} & {Dist.} & {Base Loss}
      & {Gen.-PPL}$\downarrow$ & {Entropy}$\uparrow$
      & {Gen.-PPL}$\downarrow$ & {Entropy}$\uparrow$
      & {Gen.-PPL}$\downarrow$ & {Entropy}$\uparrow$ \\
    \midrule
    Feature$^{*}$ & L2$^{*}$ & w/o$^{*}$
      & $\mathbf{208.87}_{\pm 11.51}$ & $5.41_{\pm 0.01}$
      & $\mathbf{55.04}_{\pm 1.96}$   & $5.30_{\pm 0.02}$
      & $27.35_{\pm 0.20}$            & $5.16_{\pm 0.02}$ \\
    Logit   & L2 & w/o
      & $495.64_{\pm 24.43}$ & $5.84_{\pm 0.02}$
      & $84.66_{\pm 5.53}$   & $5.42_{\pm 0.02}$
      & $26.72_{\pm 1.65}$   & $4.96_{\pm 0.06}$ \\
    Logit   & KL & w/o
      & $295.96_{\pm 16.73}$ & $5.56_{\pm 0.02}$
      & $61.37_{\pm 2.23}$   & $5.20_{\pm 0.04}$
      & $\mathbf{21.70}_{\pm 1.12}$ & $4.75_{\pm 0.07}$ \\
    \midrule
    Feature & L2 & w/
      & $502.75_{\pm 22.49}$ & $5.62_{\pm 0.01}$
      & $96.07_{\pm 6.56}$   & $5.23_{\pm 0.02}$
      & $33.92_{\pm 5.82}$   & $4.91_{\pm 0.02}$ \\
    Logit   & L2 & w/
      & $1962.85_{\pm 52.81}$ & $5.93_{\pm 0.02}$
      & $786.83_{\pm 5.14}$   & $5.85_{\pm 0.01}$
      & $361.85_{\pm 22.07}$  & $5.80_{\pm 0.03}$ \\
    Logit   & KL & w/
      & $1943.95_{\pm 86.65}$ & $5.92_{\pm 0.02}$
      & $787.24_{\pm 4.18}$   & $5.85_{\pm 0.01}$
      & $283.49_{\pm 10.61}$  & $5.75_{\pm 0.02}$ \\
    \bottomrule
  \end{tabular}
  \label{tab:method_ablation}
\end{table}

\begin{table}[t]
  \centering
  \small
  \setlength{\tabcolsep}{4pt}
  \caption{\textbf{Queue size ablation for \method\
           at ckpt 5k.} $|\mathcal{Q}_\text{gen,real}|$ = queue size.
           Values are mean$_{\pm \mathrm{SD}}$ over 3 seeds.
           $^{*}$default configuration used in the main table.}
  \begin{tabular}{c rr rr rr}
    \toprule
    & \multicolumn{2}{c}{{$\text{NFE}=4$}}
    & \multicolumn{2}{c}{{$\text{NFE}=8$}}
    & \multicolumn{2}{c}{{$\text{NFE}=16$}} \\
    \cmidrule(lr){2-3}\cmidrule(lr){4-5}\cmidrule(lr){6-7}
     {$|\mathcal{Q}_\text{gen,real}|$}
      & {Gen.-PPL}$\downarrow$ & {Entropy}$\uparrow$
      & {Gen.-PPL}$\downarrow$ & {Entropy}$\uparrow$
      & {Gen.-PPL}$\downarrow$ & {Entropy}$\uparrow$ \\
    \midrule
       4
      & $1756.70_{\pm 130.91}$ & $6.00_{\pm 0.04}$
      & $428.03_{\pm 27.09}$   & $5.83_{\pm 0.02}$
      & $145.21_{\pm 7.13}$    & $5.67_{\pm 0.03}$ \\
       64
      & $1151.60_{\pm 105.96}$ & $5.88_{\pm 0.04}$
      & $261.55_{\pm 7.82}$    & $5.71_{\pm 0.02}$
      & $96.61_{\pm 5.44}$     & $5.58_{\pm 0.02}$ \\
       1024$^{*}$
      & $\mathbf{729.17}_{\pm 42.40}$ & $5.84_{\pm 0.02}$
      & $\mathbf{166.96}_{\pm 6.77}$  & $5.68_{\pm 0.01}$
      & $\mathbf{67.94}_{\pm 4.07}$   & $5.55_{\pm 0.03}$ \\
    \bottomrule
  \end{tabular}
  \label{tab:queue_size_ablation}
\end{table}

\begin{table}[t]
  \centering
  \small
  \setlength{\tabcolsep}{4pt}
  \caption{\textbf{Ablation for attraction--repulsion ratio (i.e., $b_i^+$ vs.\ $b_i^-$) at ckpt 5k.} 
           Values are mean$_{\pm \mathrm{SD}}$ over 3 seeds.
           $^{*}$default setting carried over from the main table.}
  \begin{tabular}{ll rr rr rr}
    \toprule
    \multicolumn{2}{c}{Ratio}
    & \multicolumn{2}{c}{$\text{NFE}=4$}
    & \multicolumn{2}{c}{$\text{NFE}=8$}
    & \multicolumn{2}{c}{$\text{NFE}=16$} \\
    \cmidrule(lr){1-2}\cmidrule(lr){3-4}\cmidrule(lr){5-6}\cmidrule(lr){7-8}
    $b_i^+$ & $b_i^-$
      & Gen.-PPL$\downarrow$ & Entropy$\uparrow$
      & Gen.-PPL$\downarrow$ & Entropy$\uparrow$
      & Gen.-PPL$\downarrow$ & Entropy$\uparrow$ \\
    \midrule
    1 & 0
      & $726.42_{\pm 30.80}$ & $5.74_{\pm 0.03}$
      & $232.23_{\pm 8.55}$  & $5.66_{\pm 0.01}$
      & $107.14_{\pm 4.17}$  & $5.60_{\pm 0.01}$ \\
    2 & 1
      & $\mathbf{691.49}_{\pm 22.36}$ & $5.82_{\pm 0.02}$
      & $219.42_{\pm 4.31}$  & $5.74_{\pm 0.00}$
      & $107.02_{\pm 3.63}$  & $5.69_{\pm 0.01}$ \\
    \midrule
    0 & 1
      & $17727.67_{\pm 368.66}$ & $6.56_{\pm 0.01}$
      & $20322.75_{\pm 610.23}$ & $6.65_{\pm 0.01}$
      & $19778.34_{\pm 752.82}$ & $6.67_{\pm 0.02}$ \\
    1 & 2
      & $16996.35_{\pm 502.84}$ & $6.53_{\pm 0.01}$
      & $19553.59_{\pm 253.84}$ & $6.62_{\pm 0.01}$
      & $19167.79_{\pm 143.71}$ & $6.64_{\pm 0.01}$ \\
    \midrule
    1$^{*}$ & 1$^{*}$
      & ${729.17}_{\pm 42.40}$ & $5.84_{\pm 0.02}$
      & $\mathbf{166.96}_{\pm 6.77}$  & $5.68_{\pm 0.01}$
      & $\mathbf{67.94}_{\pm 4.07}$   & $5.55_{\pm 0.03}$ \\
    \bottomrule
  \end{tabular}
  \label{tab:att_rep_ablation}
\end{table}

\subsection{Drift Estimation Ablations}
\label{subsec:drift_estimation_ablations}

\textbf{Reference-set size.}
Table~\ref{tab:queue_size_ablation} varies the number of real and generated reference features used to estimate the drift field. 
Gen.-PPL improves consistently as the queue size increases: small queues give noisy drift estimates, while the default queue size of 1024 gives the best result at every reported NFE. 
Entropy stays in a comparable range, suggesting that the gain comes from a better-estimated drift direction rather than from a collapse in output diversity. 
Thus, sufficiently large reference sets are important for making the attraction--repulsion field reliable.

\textbf{Attraction--repulsion balance.}
Table~\ref{tab:att_rep_ablation} tests the anti-symmetric structure inherited from drifting~\cite{drifting}. 
The balanced field \(V_i=b_i^+-b_i^-\), which preserves the feature-space equilibrium property discussed in Corollary~\ref{cor:equilibrium}, performs best across the evaluated NFEs. 
Attraction-only or attraction-heavy variants remain usable but are consistently weaker, while repulsion-only or repulsion-heavy variants fail catastrophically with extremely high Gen.-PPL and inflated entropy. 
This shows that the attraction--repulsion balance is not just a theoretical convenience: it is necessary for a stable and useful drift estimate in discrete diffusion language models.

\subsection{Qualitative Samples}
\label{subsec:qualitative_samples}

\begin{figure}[t]
    \begin{samplebox}[MDLM]{gray}
    Upon takingHerEven so I was under sprinkler cabin and TommyFire set ablaze fire.There wasMy igies bucket here thaif weand my nose in Pac-Pac, my step  ...
    \end{samplebox}
    \begin{samplebox}[\method]{orange}
    The police were asked to register the case on Thursday after forensic examination showed blood and bruises on the body. File photo: A poison victim was found ... 
    \end{samplebox}
    \caption{\textbf{A generated example:}  MDLM (top) and \method\ (bottom) at $\text{NFE}=16$.}
    \label{fig:samples}
\end{figure}

Figure~\ref{fig:samples} shows representative generations from MDLM and \method\ at $\text{NFE}=16$. 
The MDLM sample exhibits the typical failure mode of low-budget diffusion decoding, with fragmented phrases, abrupt topic shifts, and weak local coherence. 
In contrast, \method\ produces a substantially more fluent and document-like continuation with coherent syntax and a plausible news-style structure. 
This qualitative comparison is consistent with the Gen.-PPL improvements in Tables~\ref{tab:gen_ppl_entropy} and~\ref{tab:gen_ppl_entropy_duo}.

\section{Related Work}
\textbf{Distillation for discrete diffusion language models.}
Distillation methods such as SDTT~\cite{SDTT}, Di4C~\cite{di4c}, and DiDi-Instruct~\cite{didi-instruct} train DDLM-based models for low-NFE generation using pretrained teachers. 
They remain closely related to our setting, but solve a different optimization problem: their goal is to distill a teacher into a few-step generator, whereas we refine the same starting DDLM checkpoint with a drifting objective under matched additional training budgets and NFEs. 
We therefore report distillation results as contextual references rather than controlled baselines.

\textbf{Flow-based formulations for discrete sequences.}
Another line of work studies flow-based formulations for discrete or simplex-valued sequence generation~\cite{fmlm,discrete_flow_maps,roos2026categorical,stark2024dirichlet,fsdfm}. 
Rather than refining an existing denoising model and sampler, they define generation through continuous denoising dynamics, flow maps, simplex-valued probability paths, or discrete flow-matching objectives. 
We instead keep the DDLM checkpoint and sampler fixed, and test whether a drifting objective improves the same model under matched training budgets and NFEs.

\section{Conclusion}
\label{sec:conclusion}
We introduced drifting objectives for refining discrete diffusion language models. By lifting categorical predictions to soft-token features, our formulation makes feature-space drifting trainable for DDLMs while preserving the attraction--repulsion structure of continuous drifting. 
Controlled experiments with OWT show that this objective improves fixed-NFE generation quality over matched continuation baselines and transfers beyond masked diffusion to a uniform-state backbone. 
Our results suggest that drifting is a practical training objective for improving existing DDLMs without changing their sampler or relying on specialized distillation.

\textbf{Discussions.}
The main limitation of \method\ is that its effectiveness depends on the quality of the frozen semantic feature space used to estimate drift. 
In addition, our experiments focus on unconditional generation with DDLM backbones; extending the same objective to conditional generation, instruction-following settings, or larger-scale diffusion language models remains an important direction for future work.

\newpage

\section*{Acknowledgement}
This work was partially supported by JSPS KAKENHI Grant Number 25H01137 and JST K Program Japan Grant Number JPMJKP24C3.

\bibliographystyle{unsrtnat}
\bibliography{custom}

\newpage
\appendix

\section{Broader Impact}
\label{app:broader_impact}

This work studies training objectives for improving discrete diffusion language models, and therefore shares the broader impacts of generative language modeling. Improved generation quality at fixed inference budgets may make non-autoregressive text generation more practical and efficient, which could benefit applications requiring lower latency or reduced sampling cost. At the same time, stronger language generators can also be misused to produce misleading, low-quality, or harmful text at scale. Our method does not introduce new data sources, deployment mechanisms, or safety filters, and our experiments are limited to standard unconditional generation benchmarks. Responsible deployment of models trained with drifting objectives should therefore follow existing best practices for language-model evaluation, including bias, toxicity, memorization, and misuse assessments before use in user-facing systems.

\section{Licenses}
\label{app:licenses}

We use publicly available datasets, checkpoints, and evaluation models in accordance with their released licenses and terms. Table~\ref{tab:licenses} summarizes the main external assets used in this work. We do not redistribute the original datasets or pretrained checkpoints; our release will include only our code, configuration files, and instructions for obtaining the external resources from their original providers.

\begin{table}[h]
    \centering
    \setlength{\tabcolsep}{4pt}
    \caption{\textbf{Main external assets used in this work.}}
    \label{tab:licenses}
    \begin{tabular}{lll}
        \toprule
        Asset & Use in this work & License / terms \\
        \midrule
        OpenWebText~\cite{owt} 
        & continual training and evaluation 
        & CC0-1.0 \\

        MDLM~\cite{sahoo2024simple} 
        & released masked-diffusion checkpoint and codebase 
        & Apache-2.0 \\

        DUO~\cite{duo} 
        & released uniform-state diffusion checkpoint and codebase 
        & Apache-2.0 \\

        GPT-2 Large~\cite{gpt2} 
        & Gen.-PPL evaluator 
        & Modified MIT \\

        \bottomrule
    \end{tabular}
\end{table}

\section{Theoretical Analysis of \method}
\label{sec:theory}
We give three results that justify the feature-space drifting objective used by \method. The first shows that the soft-token lift makes drifting differentiable for categorical sequence generators. The second shows that the resulting fixed-point loss directly follows the feature-space drift. The third establishes that the equilibrium property of anti-symmetric drifting is preserved in the discrete setting.
We provide formal statements and proofs for the properties summarized in Sec.~\ref{subsec:objective_properties}.

Throughout this section, we omit the sample index when it is clear from context. For logits $\ell\in\mathbb{R}^{L\times |\mathcal V|}$, we write
\[
    p=\operatorname{softmax}(\ell), 
    \qquad 
    \tilde e = pE,
    \qquad
    h(\ell)=\phi(\tilde e),
\]
where the softmax is applied row-wise over the vocabulary dimension, $E$ is the token embedding matrix, and $\phi$ is a frozen semantic encoder. Given a feature-space drift $V$, the drifting target and objective are
\[
    h^\star = \operatorname{sg}\!\left(h(\ell)+\alpha V\right),
    \qquad
    \mathcal L_{\mathrm{drift}}(\ell;V)
    =
    \frac12
    \left\|h(\ell)-h^\star\right\|_2^2,
\]
where $\alpha>0$ is the drift scale and $\operatorname{sg}(\cdot)$ denotes stop-gradient.

\subsection{Soft-token features make drifting differentiable}
\label{subsec:theory_soft_token}

The first issue in applying drifting to text is that hard token samples are not differentiable with respect to logits. The soft-token lift resolves this by replacing a hard token embedding with the expected embedding under the model distribution.

\begin{proposition}[Differentiable soft-token lift]
\label{prop:soft_token_lift}
Assume that the frozen encoder $\phi$ is differentiable with respect to its input embeddings. Then
\[
    h(\ell)=\phi(\operatorname{softmax}(\ell)E)
\]
is differentiable with respect to $\ell$. Consequently, for any differentiable feature-space objective $R(h)$,
\[
    \nabla_{\ell} R(h(\ell))
    =
    J_h(\ell)^\top \nabla_h R(h),
\]
where $J_h(\ell)$ is the Jacobian of $h$ with respect to $\ell$.
The Jacobian $J_h(\ell)$ pulls the feature-space gradient back to logit space.
\end{proposition}

\textbf{Proof.}
The row-wise softmax map
\[
\ell \mapsto p=\mathrm{softmax}(\ell)
\]
is differentiable with respect to $\ell$. 
The soft-token embedding map
\[
p \mapsto \tilde e=pE
\]
is linear, and hence differentiable. 
By assumption, the frozen encoder $\phi$ is differentiable with respect to its input embeddings. 
Therefore, the composition
\[
h(\ell)
=
\phi(\mathrm{softmax}(\ell)E)
\]
is differentiable with respect to $\ell$.

Now let $R$ be any differentiable objective defined on the feature representation $h$. 
Since $R(h(\ell))$ is a composition of differentiable maps, the chain rule gives
\[
\nabla_{\ell} R(h(\ell))
=
J_h(\ell)^\top \nabla_h R(h),
\]
where $J_h(\ell)$ is the Jacobian of $h$ with respect to $\ell$. 
Thus, gradients of feature-space objectives can be pulled back to the generator logits through the soft-token lift.
\qed

\textbf{Interpretation.}
Proposition~\ref{prop:soft_token_lift} identifies the key bridge from continuous drifting to discrete text. Although the drift loss is defined in a semantic feature space, the soft-token lift makes this loss differentiable with respect to the generator logits. Thus, the feature-space drift can update a categorical sequence generator by standard backpropagation. In contrast, hard token selection would break this path and turn the drift target into a non-differentiable evaluation signal rather than a trainable objective.

\subsection{Connecting feature-space drifting to logits}
\label{subsec:theory_drift_gradient}

The fixed-point loss below is the same local mechanism used in continuous drifting: a generated feature is trained toward a stop-gradient target shifted by the drift field. The nontrivial point in our setting is that the optimized variables are not continuous samples, but token logits. Using the soft-token lift (Prop.~\ref{prop:soft_token_lift}), we show that the feature-space drifting signal induces a well-defined update direction for the logits of a categorical generator.

\begin{proposition}[Logit-space pullback of the drift]
\label{prop:drift_gradient}
Fix a drift vector $V$ and define
\[
    h^\star=\operatorname{sg}(h+\alpha V),
    \qquad
    \mathcal L_{\mathrm{drift}}
    =
    \frac12\|h-h^\star\|_2^2.
\]
Then the feature-space gradient is
\[
    \nabla_h \mathcal L_{\mathrm{drift}}=-\alpha V.
\]
Moreover, when $h=h(\ell)$ is the soft-token feature map from Proposition~\ref{prop:soft_token_lift},
\[
    \nabla_{\ell}\mathcal L_{\mathrm{drift}}
    =
    -\alpha J_h(\ell)^\top V.
\]
Thus, the continuous drift field is transferred to the categorical generator as a logit-space update through the Jacobian of the soft-token feature map.
\end{proposition}

\textbf{Proof.}
Since \(h^\star=\mathrm{sg}(h+\alpha V)\), the target is treated as constant when differentiating with respect to \(h\). Therefore,
\[
\nabla_h \mathcal L_{\mathrm{drift}}
=
\nabla_h \frac12\|h-h^\star\|_2^2
=
h-h^\star
=
h-\mathrm{sg}(h+\alpha V)
=
-\alpha V .
\]
Now suppose \(h=h(\ell)\). By Proposition~\ref{prop:soft_token_lift}, \(h(\ell)\) is differentiable with respect to \(\ell\). Applying the chain rule gives
\[
\nabla_\ell \mathcal L_{\mathrm{drift}}
=
J_h(\ell)^\top \nabla_h \mathcal L_{\mathrm{drift}}
=
-\alpha J_h(\ell)^\top V .
\]
This proves the stated logit-space pullback of the feature-space drift.
\qed

\textbf{Local effect on generated features.}
We next check that pulling the drift back to logits does not destroy its intended feature-space effect.
Let $\ell^+ = \ell-\gamma\nabla_\ell \mathcal L_{\mathrm{drift}}$ be one gradient descent step with step size $\gamma>0$. A first-order expansion gives
\[
    \left\langle h(\ell^+)-h(\ell),\, V\right\rangle
    =
    \gamma\alpha \|J_h(\ell)^\top V\|_2^2
    + O(\gamma^2).
\]
Therefore, for sufficiently small $\gamma$, the logit update moves the generated feature in a direction that has positive alignment with the drift, whenever $J_h(\ell)^\top V\neq 0$.

\textbf{Interpretation.}
The identity $\nabla_h \mathcal L_{\mathrm{drift}}=-\alpha V$ mirrors the fixed-point construction of continuous drifting. The contribution here is the connection to discrete generators: because the feature map is built from soft token predictions, the same drift-following signal can be pulled back to token logits as $-\alpha J_h(\ell)^\top V$. In this sense, the objective is not merely a feature-space matching loss; it is a trainable way to apply continuous drifting dynamics to categorical sequence models.

\subsection{Inherited equilibrium under anti-symmetric drift}
\label{subsec:theory_equilibrium}

The original drifting formulation relies on an anti-symmetric attraction--repulsion field so that the drift vanishes when the data and model distributions match. 
We here verify that the same fixed-point property is inherited by our discrete objective after replacing continuous samples with soft-token features.

Let $P_{\mathrm{data}}^\phi$ denote the distribution of real-text features $u=\phi(E[x])$, and let $P_{\mathrm{model}}^\phi$ denote the distribution of generated soft-token features $h(\ell)=\phi(\mathrm{softmax}(\ell)E)$. Let $\bar V(h;P,Q)$ be the population drift field in this frozen feature space.

\begin{corollary}[No drift signal at feature-space equilibrium]
\label{cor:equilibrium}
Assume the feature-space drift operator is anti-symmetric:
\[
    \bar V(h;P,Q)=-\bar V(h;Q,P)
    \qquad
    \text{for all } h,P,Q.
\]
If $P_{\mathrm{data}}^\phi=P_{\mathrm{model}}^\phi$, then
\[
    \bar V(h;P_{\mathrm{data}}^\phi,P_{\mathrm{model}}^\phi)=0.
\]
Consequently, the stop-gradient target satisfies
\[
    h^\star
    =
    \mathrm{sg}(h+\alpha \bar V)
    =
    \mathrm{sg}(h),
\]
and therefore
\[
    \mathcal L_{\mathrm{drift}}
    =
    \frac12\|h-h^\star\|_2^2
    =
    0,
    \qquad
    \nabla_h\mathcal L_{\mathrm{drift}}=0,
    \qquad
    \nabla_\ell\mathcal L_{\mathrm{drift}}=0.
\]
\end{corollary}

\textbf{Proof.}
Since \(P_{\mathrm{data}}^\phi=P_{\mathrm{model}}^\phi\), write this common feature distribution as \(P\). 
By anti-symmetry,
\[
    \bar V(h;P,P)=-\bar V(h;P,P).
\]
Hence \(2\bar V(h;P,P)=0\), so \(\bar V(h;P,P)=0\). 
Therefore,
\[
    \bar V(h;P_{\mathrm{data}}^\phi,P_{\mathrm{model}}^\phi)=0.
\]

Substituting this into the target definition gives
\[
    h^\star
    =
    \mathrm{sg}(h+\alpha \bar V)
    =
    \mathrm{sg}(h).
\]
Since \(\mathrm{sg}(h)\) has the same value as \(h\), we have
\[
    \mathcal L_{\mathrm{drift}}
    =
    \frac12\|h-\mathrm{sg}(h)\|_2^2
    =
    0.
\]
Moreover, using Proposition~\ref{prop:drift_gradient} with \(V=0\),
\[
    \nabla_h \mathcal L_{\mathrm{drift}}=0,
    \qquad
    \nabla_\ell \mathcal L_{\mathrm{drift}}
    =
    -\alpha J_h(\ell)^\top 0
    =
    0.
\]
Thus, at feature-space equilibrium, the drifting objective injects no feature-space or logit-space learning signal.
\qed

\textbf{Interpretation.}
This result is inherited from the equilibrium property of continuous drifting. Its role is to show that our discrete objective does not introduce a spurious learning signal at feature-space equilibrium. The statement is deliberately made in feature space: matching $P_{\mathrm{data}}^\phi$ and $P_{\mathrm{model}}^\phi$ does not imply equality of logits or token distributions, since the feature map may be many-to-one. What it guarantees is that, in the geometry where the drift is defined, the objective stops pushing once the model-induced feature distribution matches the data-induced feature distribution.

\textbf{Summary.}
Together, these results justify the core design of \method. The soft-token lift makes feature-space drifting differentiable for categorical generators; the fixed-point loss follows the computed drift direction exactly in feature space; and anti-symmetry ensures that the learning signal vanishes at equilibrium. These results do not claim global convergence of the nonconvex generator, but they characterize the local mechanism and fixed-point structure of the proposed objective.

\section{Theoretical Analysis of Mirror-Teacher Objectives}
\label{app:mirror_theory}

We provide the theoretical details for the mirror-teacher alternatives used in Sec.~\ref{subsec:mirror_alternative}. 
The first result shows that the mirror teacher is the KL-proximal simplex update induced by a logit-space direction. 
The second shows that this direction locally improves alignment with the feature-space drift. 
The third verifies that the equilibrium property of anti-symmetric drifting is preserved by the mirror-teacher loss.

Throughout, for a current logit tensor $\ell$, we write
\[
    p = \mathrm{softmax}(\ell),
    \qquad
    g = \nabla_{\ell}\!\left\langle h(\ell), \mathrm{sg}(V)\right\rangle,
    \qquad
    p^\star = \mathrm{softmax}(\mathrm{sg}(\ell+\eta g)),
\]
where $h(\ell)$ is the soft-token semantic feature, $V$ is the feature-space drift, and $\mathrm{sg}(\cdot)$ denotes stop-gradient.

\subsection{Mirror teacher as a simplex-constrained update}
\label{subsec:theory_mirror}

Our mirror teacher has the closed form
\[
    p^\star_v \propto p_v \exp(\eta g_v),
\]
which is standard in mirror descent and exponentiated-gradient updates on the
simplex~\cite{mirror_descent, exponentiated_gradient}.
The following proposition shows that this teacher is not an ad hoc construction:
it is the unique KL-regularized update that moves the current distribution in
direction $g$ while remaining on the simplex.

\begin{proposition}[Variational form of the mirror teacher]
\label{prop:mirror_teacher}
Let $p \in \Delta^{|\mathcal V|-1}$ and $g \in \mathbb{R}^{|\mathcal V|}$.
For any $\eta > 0$, the distribution
\[
    p^\star_v = \frac{p_v \exp(\eta g_v)}{\sum_{u} p_u \exp(\eta g_u)}
\]
is the unique solution of
\[
    \arg\max_{q \in \Delta^{|\mathcal V|-1}}
    \left\{ \langle g, q \rangle - \frac{1}{\eta}\operatorname{KL}(q \,\|\, p) \right\}.
\]
Equivalently, $p^\star = \operatorname{softmax}(\ell + \eta g)$ whenever $p=\operatorname{softmax}(\ell)$.
\end{proposition}

\textbf{Proof.}
Consider the Lagrangian
\[
\mathcal J(q,\lambda)
=
\sum_v q_v g_v
-
\frac{1}{\eta}\sum_v q_v\log\frac{q_v}{p_v}
+
\lambda\left(\sum_v q_v-1\right).
\]
Differentiating with respect to \(q_v\) gives
\[
g_v
-
\frac{1}{\eta}
\left(
\log\frac{q_v}{p_v}+1
\right)
+
\lambda
=
0.
\]
Rearranging,
\[
q_v
=
C\,p_v\exp(\eta g_v),
\]
where \(C>0\) is a constant independent of \(v\). Enforcing the simplex constraint \(\sum_v q_v=1\) yields
\[
q_v^\star
=
\frac{p_v\exp(\eta g_v)}
{\sum_u p_u\exp(\eta g_u)}.
\]
Since the objective is linear in \(q\) plus the strictly concave term
\(-\frac{1}{\eta}\operatorname{KL}(q\|p)\), it is strictly concave on the simplex interior whenever \(p_v>0\) for all \(v\). Hence this stationary point is the unique maximizer. Finally, if \(p=\mathrm{softmax}(\ell)\), then
\[
q_v^\star
=
\frac{\exp(\ell_v)\exp(\eta g_v)}
{\sum_u \exp(\ell_u)\exp(\eta g_u)}
=
\mathrm{softmax}(\ell+\eta g)_v.
\]
Therefore \(p^\star=\mathrm{softmax}(\ell+\eta g)\), as claimed.
\qed

\textbf{Interpretation.}
Proposition~\ref{prop:mirror_teacher} formalizes the role of the mirror step in this alternative mirror formulation of \method.
A Euclidean additive update would generally leave the simplex, whereas the mirror
teacher is the KL-proximal update naturally associated with categorical outputs.

\subsection{Feature-space drift induces local semantic improvement}
\label{subsec:theory_local}

The drift field is computed in semantic feature space, but the generator is
updated in logit space.
The next result shows that this mismatch is locally well behaved: 
the logit-space direction used by mirror-formulation of \method\ provably improves alignment with the desired feature-space drift.

Define the local alignment objective
\[
    \Psi(\ell;V) = \left\langle h(\ell), \mathrm{sg}(V)\right\rangle.
\]
By construction,
\[
    g = \nabla_{\ell}\Psi(\ell;V).
\]

\begin{proposition}[Local ascent in semantic alignment]
\label{prop:mirror_local_ascent}
Assume that $\Psi(\ell;V)$ is $L$-smooth in $\ell$ for fixed $V$.
Let
\[
    \tilde{\ell} = \ell + \eta g,
    \qquad
    g = \nabla_{\ell}\Psi(\ell;V).
\]
Then, for any $0 < \eta < 2/L$,
\[
    \Psi(\tilde{\ell};V)
    \ge \Psi(\ell;V) + \eta\left(1-\frac{L\eta}{2}\right)\|g\|_2^2.
\]
In particular, if $g \neq 0$, then the teacher logits $\tilde{\ell}$ strictly
improve semantic alignment with the drift direction for sufficiently small $\eta$.
\end{proposition}

\textbf{Proof.}
Since $\Psi(\cdot;V)$ is $L$-smooth in $\ell$, for any perturbation $\Delta$,
\[
    \Psi(\ell+\Delta;V)
    \ge
    \Psi(\ell;V)
    +
    \left\langle \nabla_\ell \Psi(\ell;V), \Delta \right\rangle
    -
    \frac{L}{2}\|\Delta\|_2^2 .
\]
Set $\Delta=\eta g$, where $g=\nabla_\ell\Psi(\ell;V)$. Then
\[
\begin{aligned}
    \Psi(\ell+\eta g;V)
    &\ge
    \Psi(\ell;V)
    +
    \eta\left\langle \nabla_\ell\Psi(\ell;V), g \right\rangle
    -
    \frac{L\eta^2}{2}\|g\|_2^2 \\
    &=
    \Psi(\ell;V)
    +
    \eta\|g\|_2^2
    -
    \frac{L\eta^2}{2}\|g\|_2^2 \\
    &=
    \Psi(\ell;V)
    +
    \eta\left(1-\frac{L\eta}{2}\right)\|g\|_2^2 .
\end{aligned}
\]
Since $0<\eta<2/L$, the coefficient
\[
\eta\left(1-\frac{L\eta}{2}\right)
\]
is positive. Therefore, if $g\neq 0$, the teacher logits
$\tilde\ell=\ell+\eta g$ strictly increase $\Psi(\ell;V)$.
\qed

\textbf{Feature-space view.}
Writing $J_h(\ell)$ for the Jacobian of $h$ with respect to $\ell$, we have
\[
    g = J_h(\ell)^\top V.
\]
Hence a first-order expansion gives
\[
    \left\langle h(\ell+\eta g)-h(\ell),\, V \right\rangle
    = \eta \|J_h(\ell)^\top V\|_2^2 + O(\eta^2),
\]
which makes the geometric role of the update explicit: mirror-based formulation of \method\ chooses the
logit-space direction that increases semantic alignment with the feature-space
drift at first order.

\textbf{Interpretation.}
Proposition~\ref{prop:mirror_local_ascent} is the main justification for transporting
drift from feature space to logits.
It shows that the mirror teacher is not merely simplex-valid; it is also locally
aligned with the semantic correction prescribed by the drift field.

\subsection{Equilibrium preservation under anti-symmetric drift}
\label{subsec:theory_mirror_equilibrium}

A central property of drifting is that the learning signal should disappear once
the model and data distributions match~\cite{drifting}.
We now show that this equilibrium property is preserved by mirror-based formulation of \method.

Let $P_{\mathrm{data}}^\phi$ and $P_{\mathrm{model}}^\phi$ denote the
distributions of semantic features induced by the data and the current model
under the frozen encoder $\phi$.
Let $\bar V(h;P,Q)$ denote the population drift field in feature space, where
$P$ supplies the attractive references and $Q$ supplies the repulsive references.

\begin{proposition}[Equilibrium under anti-symmetric drift]
\label{prop:mirror_equilibrium}
Assume that the population drift field is anti-symmetric:
\[
    \bar V(h;P,Q) = -\,\bar V(h;Q,P) \qquad \text{for all } h,P,Q.
\]
Then
\[
    \bar V(h;P,P)=0 \qquad \text{for all } h.
\]
Consequently, if $P_{\mathrm{data}}^\phi = P_{\mathrm{model}}^\phi$, then the
drift term vanishes, the induced logit-space direction satisfies $g=0$, the
mirror teacher reduces to $p^\star=p$, and the drift loss is zero:
\[
    \operatorname{KL}(p^\star \,\|\, p)=0.
\]
\end{proposition}

\textbf{Proof.}
Setting \(Q=P\) in the anti-symmetry condition gives
\[
    \bar V(h;P,P)
    =
    -\bar V(h;P,P).
\]
Hence \(2\bar V(h;P,P)=0\), and therefore
\[
    \bar V(h;P,P)=0.
\]

Now suppose \(P_{\mathrm{data}}^\phi=P_{\mathrm{model}}^\phi\). 
Let this common feature distribution be denoted by \(P\). 
Then the population drift used by the mirror-teacher objective satisfies
\[
    V
    =
    \bar V(h;P_{\mathrm{data}}^\phi,P_{\mathrm{model}}^\phi)
    =
    \bar V(h;P,P)
    =
    0.
\]
By the definition of the logit-space direction,
\[
    g
    =
    \nabla_{\ell}
    \left\langle h(\ell),\mathrm{sg}(V)\right\rangle,
\]
and since \(V=0\), we have
\[
    g
    =
    \nabla_{\ell}
    \left\langle h(\ell),0\right\rangle
    =
    0.
\]
Therefore the mirror teacher becomes
\[
    p^\star
    =
    \mathrm{softmax}(\mathrm{sg}(\ell+\eta g))
    =
    \mathrm{softmax}(\mathrm{sg}(\ell)).
\]
The stop-gradient operator does not change the value of its argument, so
\[
    p^\star
    =
    \mathrm{softmax}(\ell)
    =
    p.
\]
Thus the mirror KL loss vanishes:
\[
    \mathrm{KL}(p^\star\|p)
    =
    \mathrm{KL}(p\|p)
    =
    0.
\]
This proves that, at feature-space equilibrium, the mirror-teacher objective injects no drift signal.
\qed

\textbf{Interpretation.}
Proposition~\ref{prop:mirror_equilibrium} shows that the discrete mirror-teacher
construction preserves the fixed-point structure of drifting.
When the model-induced feature distribution matches the data-induced feature
distribution, 
mirror-based formulation of \method\ injects no spurious learning signal.
This is the discrete analogue of the equilibrium property that motivates
drifting in the continuous setting.

\subsection{Summary}
Taken together, the three results establish that mirror-based formulation of \method\ is theoretically well
aligned with its design goals: the teacher is the correct simplex-constrained
update, the chosen direction improves semantic alignment with the drift field at
first order, and the resulting learning signal vanishes at equilibrium.

\section{Additional Related Work}
\label{app:additional_related_work}

\paragraph{Inference-time efficiency for discrete diffusion language models.}
A separate line of work improves the decoding efficiency of discrete diffusion language models at inference time. These methods reduce sampling cost by shortening denoising schedules, reusing or skipping computation across denoising steps, or stopping computation once parts of the sequence have converged~\citep{luxembourg2025plan,israel2025accelerating,wei2025accelerating,ma2025dkv,liu2025dllm,wu2025fast,oba2026stopping}. For example, SureLock~\cite{oba2026stopping} reduces decoding cost by early-stopping token positions whose predictions have stabilized during iterative denoising. These approaches modify the \emph{inference} procedure or computation schedule of an already trained model.

\paragraph{Relation to our work.}
Our work addresses a different part of the pipeline: we study drifting as a training-time objective for refining discrete diffusion language models under matched sampling budgets. The semantic encoder, reference queues, and drift targets used by \method\ are used only during training; at inference time, the underlying DDLM sampler is unchanged. Thus, inference-time acceleration methods are not controlled baselines for our objective-level study. Instead, they represent a distinct direction that improves how a DDLM is sampled, whereas our work improves the model produced before sampling begins.

\section{Additional Experimental Details}
\label{app:exp_details}

\paragraph{Benchmark and continual-training setup.}
We use OpenWebText (OWT)~\cite{owt} for the main continual-training experiments. 
Unless otherwise stated, all methods start from a released checkpoint and are trained for {13k} additional global steps under the same data and compute budget. 
We focus on unconditional text generation, following prior work on diffusion language models~\cite{sahoo2024simple}, and evaluate all methods with matched numbers of function evaluations (NFEs).

\paragraph{Backbones.}
For the masked-diffusion experiments, we use the released 170M-parameter MDLM checkpoint~\cite{sahoo2024simple}\footnote{Released at \url{https://github.com/kuleshov-group/mdlm}.}.
For the uniform-state diffusion experiments, we use the released DUO checkpoint~\cite{duo}\footnote{Released at \url{https://github.com/s-sahoo/duo}.}. 
Within each backbone, all controlled comparisons share the same architecture, tokenizer, initialization, data, additional training budget, and evaluation protocol.

\paragraph{Controlled baselines.}
For each backbone, we compare against two matched baselines. 
The first is the released checkpoint evaluated without additional training. 
The second, denoted \textbf{Continuation}, starts from the same checkpoint and is trained for the same number of additional updates using only the original DDLM objective. 
These baselines isolate the effect of the drifting objective from improvements due merely to additional optimization.

\paragraph{Drifting objective variants.}
Our default method uses the direct feature-space drifting objective described in Sec.~\ref{subsec:feature_objective}. 
We also evaluate variants that add the original denoising objective and mirror-teacher alternatives described in Sec.~\ref{subsec:mirror_alternative}. 
Unless otherwise stated, the default setting uses the drifting objective alone; base-plus-drift and mirror-teacher variants are used for the formulation study.

\paragraph{Optimization and drift hyperparameters.}
The drift scale is fixed to $\alpha=1$. 
We use multi-temperature drift estimation with $\mathcal T=\{0.02,0.05,0.2\}$. 
For each temperature, we compute per-sample drift vectors, normalize them by a scalar batch-level RMS scale, and average the normalized drifts following the original drifting implementation~\cite{drifting}. 
We optimize with AdamW using a global batch size of 512, learning rate $3\times10^{-5}$ or $5\times10^{-5}$. 

\paragraph{Feature encoder.}
For each backbone, the frozen semantic encoder is a frozen copy of the same released checkpoint used to initialize the generator. 
We extract sequence-level features by mean-pooling the penultimate and final hidden layers and L2-normalizing their concatenation. 
Concretely, if $H^{(L-1)},H^{(L)}\in\mathbb{R}^{T\times d}$ are the penultimate and final hidden states, we define
\[
h
=
\operatorname{normalize}
\left(
\left[
\frac{1}{T}\sum_{t=1}^{T}H_t^{(L-1)}
;
\frac{1}{T}\sum_{t=1}^{T}H_t^{(L)}
\right]
\right)
\in\mathbb{R}^{2d}.
\]

\paragraph{Drift reference sets.}
To estimate the drift field efficiently, we build positive and negative reference sets from the current distributed micro-batch and feature queues. 
Each GPU uses a per-device batch size of 2, and training is performed on 4 GPUs of NVIDIA H100 SXM5 94GB, so each micro-step provides 8 fresh examples after cross-device all-gather. 
Generated features are computed from model predictions on corrupted inputs sampled from the underlying diffusion process, and real features are computed from the corresponding clean texts. 
For each anchor, we use the current real and generated features together with separate FIFO queues of detached real and generated features. 
The default queue size is 1024 for both queues. 
The queues store only pooled sequence-level features, not token-level representations. 
We compute $\mathcal L_{\mathrm{drift}}$ at every micro-step and accumulate gradients to reach the global batch size of 512.

\paragraph{Evaluation protocol.}
Following prior work~\cite{sahoo2024simple}, we evaluate generated sample quality using generative perplexity (Gen.-PPL) computed by pretrained GPT-2 Large~\cite{gpt2}. 
We also report entropy as a diagnostic for diversity and degeneration.

\end{document}